%% file: main.tex
\documentclass{article}
\usepackage[utf8]{inputenc}

\input{math_commands.tex}

\usepackage{hyperref}
\usepackage{url}
\usepackage{graphicx}
\usepackage{subfigure}
\usepackage{booktabs} 
\usepackage{hyperref}
\usepackage{subfigure,graphicx,dsfont,amssymb,amsmath,calc,tabulary}
\usepackage{xcolor}
\usepackage[english]{babel}
\usepackage[utf8]{inputenc}

\usepackage{amsmath}
\usepackage{amssymb}
\usepackage{amsthm}
\usepackage{fixmath}
\newtheorem{theorem}{Theorem}

\newtheorem{lemma}[theorem]{Lemma}
\usepackage{algorithm}
\usepackage{algorithmic}
\usepackage[round,semicolon]{natbib}

\title{Graph-Aided Online Multi-Kernel Learning}

\author{ Pouya M. Ghari, Yanning Shen \thanks{P.~M.~Ghari and Y.~Shen are with department of Electrical Engineering and Computer Science, University of California, Irvine, CA, USA. E-mails: pmollaeb@uci.edu, yannings@uci.edu.}}

\begin{document}
\maketitle

\begin{abstract}
Multi-kernel learning (MKL) has been widely used in function approximation tasks. The key problem of MKL is to combine kernels in a prescribed dictionary. Inclusion of irrelevant kernels in the dictionary can deteriorate accuracy of MKL, and increase the computational complexity. To improve the accuracy of function approximation and reduce the computational complexity, the present paper studies data-driven selection of kernels from the dictionary that provide satisfactory function approximations. Specifically, based on the similarities among kernels, the novel framework constructs and refines a graph to assist choosing a subset of kernels. In addition, random feature approximation is utilized to enable online implementation for sequentially obtained data. Theoretical analysis shows that our proposed algorithms enjoy tighter sub-linear regret bound compared with state-of-art graph-based online MKL alternatives. Experiments on a number of real datasets also showcase the advantages of our novel graph-aided framework.
\end{abstract}

\section{Introduction}
The need for function approximation arises in many machine learning studies including regression, classification, and reinforcement learning, see e.g., \citep{Chung2019}. This paper studies supervised function approximation where given data samples $\{(\vx_t,y_t)\}_{t=1}^T$, the goal is to find the function $f(.)$, such that the difference between $f(\vx_t)$ and $y_t$ is minimized. In this context, kernel learning methods exhibit reliable performance. Specifically, the function approximation problem becomes tractable under the assumption that $f(.)$ belongs to a reproducing kernel Hilbert space \citep{Scholkopf2001}.
In some cases, it is imperative to perform function approximation task in an online fashion. For instance, when the volume of data is large and is collected in a sequential fashion, it is impossible to store or process it in batch. Furthermore, suffering from the well-known problem of `curse of dimensionality' \citep{Bengio2006}, kernel learning methods are not suitable for sequential settings. This has motivated studies on online single kernel learning \citep{Lu2016,Bouboulis2017,Zhang2019} to address the curse of dimensionality. Specifically, approximating kernels by finite-dimensional feature representations such as random Fourier feature by \citet{Rahimi2007} and Nystr\"{o}m method by \citet{Williams2000}, function approximation task becomes scalable.

Most of prior studies rely on a pre-selected kernel, however, such selection requires prior information which may not be available. By contrast, utilizing multiple kernels in a given dictionary in lieu of a pre-selected kernel provides more flexible approach to obtain more accurate function approximations as it can learn combination of kernels \citep{Sonnenburg2006}. Enabled by the random feature approximation by \citet{Rahimi2007}, online multi-kernel learning algorithms have been developed by \citet{Sahoo2019,Shen2019,Ghari2020}.

One of the most important issues regarding multi-kernel learning is proper selection of kernels to construct a dictionary. This affects both computational complexity and accuracy of the function approximation significantly. However, selecting an appropriate kernel dictionary requires prior information and when such information is not available one solution is to include a large number of kernels in the dictionary. In this case, employing all kernels available in the dictionary may not be a feasible choice. Data-driven selection of subset of kernels in a given dictionary can alleviate the computational complexity. Furthermore, data-driven subset selection of kernels can enhance the accuracy of function approximation by pruning irrelevant kernels. Subset selection of kernels has been studied in \citet{Ghari2020} using a bipartite graph and the resulting algorithm. The goal of the present paper is to select a subset of kernels in a given dictionary at each time instant in order to alleviate the computational complexity and improve function approximation accuracy. Specifically, based on similarities between kernels a directed graph is constructed by which at each time instant a subset of kernels is selected. In this case, function approximations given by the chosen subset of kernels can be viewed as feedback collected from a graph which is called feedback graph. The resulting algorithm is called Similarity Feedback Graph based Multi-Kernel Learning (SFG-MKL). The present paper proves that the proposed SFG-MKL achieves a sub-linear regret of $\gO(T^\frac{3}{4})$ compared with $\gO(T^\frac{5}{6})$ of OMKL-GF \citep{Ghari2020}. Furthermore, an novel algorithm SFG-MKL-R is proposed which adaptively refines the structure of the feedback graph based on the observed losses. The present paper proves that the proposed SFG-MKL-R enjoys the sublinear regret of $\gO(\sqrt{T})$ whereas it may need more computations in comparison with the proposed SFG-MKL. 
Experiments on real datasets showcase the effectiveness of our proposed SFG-MKL and SFG-MKL-R in comparison with other online multi-kernel learning baselines. 

\section{Preliminaries} \label{sec:pre}
Given samples $(\vx_{1},{y}_{1}),\cdots,(\vx_{T},{y}_{T})$, with $\vx_t \in \sR^d$ and $y_t \in \sR$, the function approximation problem can be written as the following optimization problem
\begin{align}
    \min_{f \in \sH} \frac{1}{T}\sum_{t=1}^{T} {\gC(f(\vx_{t}),y_t)}+ \lambda \Omega(\|f\|_{\sH}^{2}) \label{eq:22}
\end{align} 
where $\gC(.,.)$ denotes the cost function, that is defined according to the learning task. For example, in regression task $\gC(.,.)$ can be least square function. In \eqref{eq:22}, $\lambda$ denotes the regularization coefficient and $\Omega(.)$ represents a non-decreasing function, which is used to prevent overfitting and control model complexity. Let $\kappa(\vx,\vx_{t}): {\mathbb{R}^d} \times {\mathbb{R}^d} \rightarrow \mathbb{R}$ represent a symmetric positive semidefinite basis function called kernel which measures the similarity between $\vx$ and $\vx_{t}$. %In the kernel based learning, it is assumed that $f(.)$ belongs to the reproducing Hilbert kernel space (RHKS) $\sH:= \{f|f(\vx)=\sum_{t=1}^{\infty}{{\alpha}_{t}\kappa(\vx,\vx_{t})}\}$. A kernel is reproducing if it satisfies ${\langle {\kappa(\vx,\vx_{t})},{\kappa(\vx,\vx_{{t}^{\prime}})} \rangle}_{\sH}={\kappa(\vx_{t},\vx_{{t}^{\prime}})}$ where ${\langle\cdot,\cdot\rangle}_{\sH}$ denotes vector inner product in Hilbert space. Also, the RKHS norm is defined as $\| f\|_{\sH}^{2} := \sum_{t}{\sum_{t^\prime}{\alpha_{t}\alpha_{t^\prime}{\kappa(\vx_{t},\vx_{{t}^{\prime}})}}}$.
The representer theorem states that the optimal solution of \eqref{eq:22} can be expressed as follows given finite data samples \citep{Wahba1990}
\begin{align}
	\hat{f}(\vx)=\sum_{t=1}^{T}{\alpha_{t}\kappa(\vx,\vx_{t})}:=\boldsymbol{\alpha}^{\top}\boldsymbol{\kappa}(\vx) \label{eq:23}
\end{align} 
where $\boldsymbol{\alpha}:=[{\alpha}_{1},...,{\alpha}_{T}]^\top$ denote the vector of unknown coefficients to be estimated, and $\boldsymbol{\kappa}(\vx):=[\kappa(\vx,\vx_{1}),\cdots,\kappa(\vx,\vx_{T})]^\top$. Furthermore, it can be inferred that the dimension of $\boldsymbol{\alpha}$ increases with the number of data samples $T$. This is known as `curse of dimensionality' \citep{Wahba1990} and arises as a challenge to solve \eqref{eq:22} in an online fashion.

One way to deal with the increasing number of variables to be estimated is to employ random feature (RF) approximation \citep{Rahimi2007}.  Moreover, the kernel $\kappa$ in \eqref{eq:23} is  shift-invariant which satisfy $\kappa(\vx_t,\vx_{t^\prime})=\kappa(\vx_t-\vx_{t^\prime})$.  However, relying on a pre-selected kernel often requires prior information that may not be available. To cope with this, multi-kernel learning can be exploited which learns the kernel as a combination of a sufficiently rich dictionary of kernels $\{\kappa_i\}_{i=1}^N$. The kernel combination is itself a kernel (Scholkopf, 2001). Let $\kappa_i(\vx_t-\vx_{t^\prime})$ be the $i$-th kernel in the dictionary of $N$ kernels which is absolutely integrable. In this case, its Fourier transform $\pi_{\kappa_i} (\boldsymbol{\psi})$ exists and can be viewed as probability density function (PDF) if the kernel is normalized such that $\kappa_i(\bm{0})=1$. Specifically, it can be written that
\begin{align}
    \kappa_i(\vx_{t}-\vx_{{t}^{\prime}})&= \int{\pi_{\kappa_i}(\boldsymbol{\psi}){e}^{j\boldsymbol{\psi}^{\top}(\vx_{t}-\vx_{{t}^{\prime}})}}d\boldsymbol{\psi}
	:= \E_{\pi_{\kappa_i}(\boldsymbol{\psi})}[{e}^{j\boldsymbol{\psi}^{\top}(\vx_{t}-\vx_{{t}^{\prime}})}]. \label{eq:24}
\end{align}
Let $\{\boldsymbol{\psi}_{i,j}\}_{j=1}^{D}$ be a set of vectors which are independently and identically distributed (i.i.d) samples from $\pi_{\kappa_i}(\boldsymbol{\psi})$. Hence, $\kappa_i(\vx_{t}-\vx_{{t}^{\prime}})$ can be approximated by the ensemble mean
$\hat{\kappa}_{i,c}(\vx_{t}-\vx_{{t}^{\prime}}):=\frac{1}{D}\sum_{i=1}^{D}{{e}^{j\boldsymbol{\psi}_{i,j}^{\top}(\vx_{t}-\vx_{{t}^{\prime}})}} $. Furthermore, the real part of $\hat{\kappa}_{i,c}(\vx_{t}-\vx_{{t}^{\prime}})$ also constitutes an unbiased estimator of $\kappa_i(\vx_{t}-\vx_{t^\prime})$ which can be written as
$
	\hat{\kappa}_{i}(\vx_{t}-\vx_{{t}^{\prime}}) = \rvz_{i}^{\top}(\vx_{t})\rvz_{i}(\vx_{{t}^{\prime}})
$ \citep{Rahimi2007},
where
\begin{align}
\rvz_{i}(\vx_{t}) = \frac{1}{\sqrt{D}}[& \sin(\boldsymbol{\psi}_{i,1}^\top\vx_{t}),\cdots,\sin(\boldsymbol{\psi}_{i,D}^\top\vx_{t}), \cos(\boldsymbol{\psi}_{i,1}^\top\vx_{t}), \cdots,\cos(\boldsymbol{\psi}_{i,D}^\top\vx_{t})]. \nonumber
\end{align}
Replacing $\kappa_i(\vx,\vx_{t})$ with $\hat{\kappa}_{i}(\vx-\vx_{t})$, $\hat{f}(\vx)$ in \eqref{eq:23} can be approximated as
\begin{align}
    \hat{f}_{\text{RF},i}(\vx) &= \sum_{t=1}^{T}{\alpha_{t}\hat{\kappa}_{i}(\vx-\vx_{t})} = \sum_{t=1}^{T}{\alpha_{t}\rvz_{i}^{\top}(\vx)\rvz_{i}(\vx_{t})} = \vtheta_i^\top \rvz_{i}(\vx) \label{eq:26}
\end{align}
where $\vtheta_i \in \sR^{2D}$ is a vector whose dimension does not increase with the number of data samples. Therefore, utilizing RF approximation can make the function approximation problem amenable for online implementation. Furthermore, the loss of the $i$-th kernel can be calculated as
\begin{align}
\gL(\vtheta_{i}^\top \rvz_i(\vx_t),y_t) = \gC(\vtheta_{i}^\top \rvz_i(\vx_t),y_t) + \lambda \Omega(\|\vtheta_{i}\|^2). \label{eq:4}
\end{align}
Moreover, when multiple kernels are employed, function approximation can be performed by functions in the form $f(\vx) = \sum_{i=1}^N{\Bar{w}_{i}f_i(\vx)}$ where $\sum_{i=1}^N{\Bar{w}_{i}}=1$ (Scholkopf, 2001). Also, $f_i(\vx) \in \sH_i$ where $\sH_i$ is an RKHS induced by the kernel $\kappa_i$. Replacing $f_i(\vx)$ with $\hat{f}_{\text{RF},i}(\vx)$, the function $f(\vx)$ can be approximated as
\begin{align}
    \hat{f}_{\text{RF}}(\vx) = \sum_{i=1}^{N}{\Bar{w}_{i}\hat{f}_{\text{RF},i}(\vx)}, \sum_{i=1}^N{\Bar{w}_{i}}=1. \label{eq:27}
\end{align}
%Thus, the function approximation problem can be written as \citep{Shen2019}
%\begin{subequations} \label{eq:28}
%\begin{align}
    %\min_{\Bar{w}_{i}, \vtheta_i} & \sum_{t=1}^{T}{ \sum_{i=1}^{N}{ \Bar{w}_{i}\gL(\vtheta_{i}^\top \rvz_i(\vx_t),y_t) }} \label{eq:28a} \\
    %\text{s.t.} & \sum_{i=1}^N{\Bar{w}_{i}}=1, \Bar{w}_{i} \ge 0
%\end{align}
%\end{subequations}
In this context, online convex optimization methods can be utilized to find and update $\{\Bar{w}_{i}\}_{i=1}^N$, $\{\vtheta_i\}_{i=1}^N$ upon receiving new datum $\vx_t$ at each time instant $t$ \citep{Sahoo2019,Shen2019,Ghari2020}. However, using all kernels in the dictionary may not be computationally efficient especially when the dictionary of kernels is so large. Furthermore, by pruning the kernels that do not provide approximations with satisfactory accuracy, the performance of function approximation can be improved in terms of accuracy. Therefore, in this paper, we propose a novel algorithm to choose a subset of kernels at each time instant instead of using all kernels in the dictionary.

\section{Online Multi-Kernel Learning with Feedback Graph}
%The present section constructs a feedback graph to select a subset of kernels for function approximation. 

The present section first introduces a disciplined way to construct feedback graph based on kernel similarities. Furthermore, a novel online MKL algorithm is developed to select a subset of kernels based on the constructed feedback graph, which is proved to obtain sub-linear regret.

\subsection{Feedback Graph Construction} \label{fg}
Selecting a subset of kernels for function approximation can decrease the computational complexity of the learning task considerably, especially when the number of kernels in a given dictionary is large. In this case, evaluation of every kernel based learner may arise as a bottleneck. Therefore, this paper aims at taking into account the similarities between existing kernels in the given dictionary to avoid unnecessary computations. To this end, we will construct a graph which models the similarities between kernels. At each time instant $t$, a subset of kernels is chosen based on the graph. %Similar kernels may provide similar approximations and as a result choosing a subset of unsimilar kernels can avoid unnecessary computations. 

%In order to make the function approximation scalable, the random feature approximation is leveraged, where RF vector $\rvz_i(\vx_t)$ depends on the PDF $\pi_{\kappa_i}(\boldsymbol{\psi})$. 
The similarity between two shift invariant kernels $\kappa_i$ and $\kappa_j$ is measured through %$\KL ( \pi_{\kappa_i}(\boldsymbol{\psi}) \Vert \pi_{\kappa_j}(\boldsymbol{\psi}) )$ which denote the Kullback-Leibler (KL) divergence between $\pi_{\kappa_i}(\boldsymbol{\psi})$ and $\pi_{\kappa_j}(\boldsymbol{\psi})$ and can be defined as
the function $\Delta ( \kappa_i(\boldsymbol{\rho}) , \kappa_j(\boldsymbol{\rho}) ), \boldsymbol{\rho} \in \mathbb{R}^d$ which is defined as
\begin{align}
%\KL ( \pi_{\kappa_i}(\boldsymbol{\psi}) \Vert \pi_{\kappa_j}(\boldsymbol{\psi}) ) \!=\! \int \pi_{\kappa_i}(\boldsymbol{\psi}) \log \frac{\pi_{\kappa_i}(\boldsymbol{\psi})}{\pi_{\kappa_j}(\boldsymbol{\psi})} d \boldsymbol{\psi}. \label{eq:1}
\Delta ( \kappa_i(\boldsymbol{\rho}) , \kappa_j(\boldsymbol{\rho}) ) = \int | \kappa_i(\boldsymbol{\rho}) - \kappa_j(\boldsymbol{\rho}) |^2 d \boldsymbol{\rho}. \label{eq:1}
\end{align}
As $\Delta ( \kappa_i(\boldsymbol{\rho}) , \kappa_j(\boldsymbol{\rho}) )$ decreases, kernels $\kappa_i(\boldsymbol{\rho})$ and $\kappa_j(\boldsymbol{\rho})$ are considered to be more similar. The following Lemma states that the function $\Delta ( \kappa_i(\boldsymbol{\rho}) , \kappa_j(\boldsymbol{\rho}) )$ exists for each pair of absolutely integrable kernels.
\begin{lemma} \label{lem:5}
Under the assumption that kernels $\{\kappa_i\}_{i=1}^{N}$ are absolutely integrable, bounded and normalized such that $\kappa_i(\mathbf{0})=1$, $\forall i: 1 \le i \le N$, the function $\Delta ( \kappa_i(\boldsymbol{\rho}) , \kappa_j(\boldsymbol{\rho}) )$ is bounded and exists for each pair of kernels $\kappa_i(.)$ and $\kappa_j(.)$.
\end{lemma}
\begin{proof}
Since kernels $\{\kappa_i(\boldsymbol{\rho})\}_{i=1}^{N}$ are assumed to be bounded as $0 \le \kappa_i(\boldsymbol{\rho}) \le 1$, $\forall \boldsymbol{\rho}, \forall i: 1 \le i \le N$, it can be concluded that $| \kappa_i(\boldsymbol{\rho}) - \kappa_j(\boldsymbol{\rho}) |^2 \le | \kappa_i(\boldsymbol{\rho}) - \kappa_j(\boldsymbol{\rho}) |$. Thus, it can be inferred that
\begin{align}
    \int | \kappa_i(\boldsymbol{\rho}) - \kappa_j(\boldsymbol{\rho}) |^2 d\boldsymbol{\rho} \le \int | \kappa_i(\boldsymbol{\rho}) - \kappa_j(\boldsymbol{\rho}) | d\boldsymbol{\rho}. \label{eq:21}
\end{align}
Furthermore, based on the Triangular inequality, it can be written that
\begin{align}
    & \int | \kappa_i(\boldsymbol{\rho}) - \kappa_j(\boldsymbol{\rho}) | d\boldsymbol{\rho} \le \int | \kappa_i(\boldsymbol{\rho}) | d\boldsymbol{\rho} + \int | \kappa_j(\boldsymbol{\rho}) | d\boldsymbol{\rho}. \label{eq:25}
\end{align}
Based on \eqref{eq:21}, \eqref{eq:25} and the fact that kernels are absolutely integrable, we can conclude that the function $\Delta ( \kappa_i(\boldsymbol{\rho}) , \kappa_j(\boldsymbol{\rho}) )$ is bounded and exists for each pair of kernels $\kappa_i(.)$ and $\kappa_j(.)$.
\end{proof}

Furthermore, the following lemma states that the average difference between function approximations given by each pair of kernels is bounded above in accordance with the function $\Delta(.,.)$ defined in \eqref{eq:1}.
\begin{lemma} \label{lem:6}
Let $C_m:=\max_i \sum_{t=1}^{T}{|\alpha_{i,t}|^2}$ where $\{\alpha_{i,t}\}_{t=1}^T$ are weights for \eqref{eq:23} associated with the $i$-th kernel $\kappa_i(.)$. Also, let $\vx$ is bounded as $\|\vx\| \le 1$ and kernels are absolutely integrable. Then, the average difference between function approximations given by $\kappa_i(.)$ and $\kappa_j(.)$ is bounded above as
\begin{align}
    & \frac{1}{\gU_d}\int |\hat{f}_i(\vx) - \hat{f}_j(\vx)|^2 d\vx \le \frac{2C_m}{\gU_d} \sum_{t=1}^T{\left(\Delta(\kappa_i(\vx-\vx_t),\kappa_j(\vx-\vx_t))+2\gU_d\right)} \label{eq:29}
\end{align}
where $\hat{f}_i(\vx)$ denotes the function approximation given by $\kappa_i(.)$ as in \eqref{eq:23} and $\gU_d$ represents $d$-dimensional Euclidean unit norm ball volume.
\end{lemma}
\begin{proof}
See Appendix \ref{ap:C}.
\end{proof}
It can be seen from  Lemma \ref{lem:6}  that the function $\Delta(\kappa_i(\boldsymbol{\rho}),\kappa_j(\boldsymbol{\rho}))$ can be used as a criterion to measure the similarity between kernels $\kappa_i(.)$ and $\kappa_j(.)$ without knowing data samples $\{\vx_t\}_{t=1}^T$. This helps reduce computational complexity of the function approximation since similarity among kernels in the dictionary can be measured offline before observing data samples.

Let $\gG :=(\gV,\gE)$ be a directed graph with vertex $v_i \in \mathcal{V}$ which represents the $i$-th kernel $\kappa_i$. There is an edge from $v_i$ to $v_j$ i.e. $(i,j) \in \gE$ if $\Delta ( \kappa_i(\boldsymbol{\rho}), \kappa_j(\boldsymbol{\rho}) ) \le \gamma_{i}$ where $\gamma_{i}$ is a threshold  for $v_i$. Furthermore, in this case, there is a self-loop for each $v_i \in \gV$. Since a subset of function approximations associated with kernels will be chosen using the graph $\gG$, the chosen subset of function approximations can be viewed as feedback collected from the graph $\gG$ and as a result the graph $\gG$ is called \emph{feedback graph}. Let $\sN_{i}^{\text{out}}$ denote the out-neighborhood set of $v_i$ which means $j \in \sN_{i}^{\text{out}}$ if $(i,j) \in \gE$. Also, let $\sN_{i}^{\text{in}}$ denote the in-neighborhood set of $v_i$ which means $j \in \sN_{i}^{\text{in}}$ if $(j,i) \in \gE$. %$\KL ( \pi_{\kappa_i}(\boldsymbol{\psi}) \Vert \pi_{\kappa_j}(\boldsymbol{\psi}) ) \le \gamma_{i}$ where $\gamma_{i}$ is a threshold  for $v_i$. 
Specifically in order to restrict the number of out-neighbors for each node to $M$, the value of $\gamma_i$ is obtained as
\begin{align}
    \gamma_i = \arg \max_{\gamma} \{\gamma | |\sN_{i}^{\text{out}}| = M\}. \label{eq:28}
\end{align}
Note that $M$ is a preselected parameter in the algorithm and increasing the value of $M$ increases the connectivity of the feedback graph. At each time instant, one of the nodes are drawn and the function approximation is carried out using the combination of a subset of kernels which are out-neighbors of the chosen node. Therefore, increase in $M$ can increase the exploration in the approximation task while it increases the computational complexity. The feedback graph construction procedure is summarized in \Algref{alg:3}.
\begin{algorithm}[tb]
	\caption{Feedback Graph Construction}
	\label{alg:3}
	\begin{algorithmic}
		\STATE {\bfseries Input:}{Kernels $\kappa_{i}$, $i=1,...,N$. }
		\FOR{$i=1,...,N$}
		\STATE Obtain $\gamma_i$ via \eqref{eq:28}.
		\STATE Append $(i,j)$ to $\gE$ if $\Delta ( \kappa_i(\boldsymbol{\rho}), \kappa_j(\boldsymbol{\rho}) ) \le \gamma_i$.
		\ENDFOR
	\end{algorithmic}
\end{algorithm}

In this paper, the goal is to find the best kernel in the hindsight to minimize the cumulative regret. To this end, proposed algorithms need performing sufficient exploration such that algorithms can evaluate all kernels to find the best kernel in the hindsight. By choosing a subset of kernels at each time instant, proposed algorithms perform exploration to find the best kernel in the hindsight. Specifically, using the feedback graph $\gG$, the algorithm will perform exploration to find the subset of kernels which include the best kernel in the hindsight and similar kernels to the best one based on the value of the function $\Delta(.,.)$. In the meantime, the proposed algorithms learn the optimal combination for each subset of kernels to minimize the cumulative loss. In what follows, the proposed algorithmic framework is presented to choose one of nodes from $\gG$ such that a subset of kernels is selected.

\subsection{Kernel Selection}
The present section studies how to select a subset of kernels  using the feedback graph $\gG$ and prior observations of losses associated with kernels. Assume that each kernel is associated with a set of weights $\{w_{i,t}\}_{i=1}^N$ where $w_{i,t}$ is the weight associated with the $i$-th kernel $\kappa_i$. The weight $w_{i,t}$ indicates the accuracy of the function approximation given by the $\kappa_i$ at time $t$ and its value can be updated when more and more information is being revealed. Furthermore, a set of weights $\{u_{i,t}\}_{i=1}^{N}$ is assigned to $\gV$ such that $u_{i,t}$ is the weight associated with $v_i \in \gV$ at time instant $t$, which indicates the accuracy of function approximation when the node $v_i$ is drawn. In order to choose a subset of kernels at time $t$, one of the vertices in $\gV$ is drawn according to the  probability mass function (PMF) $p_t$
\begin{align}
	p_{i,t} = (1-\xi)\frac{u_{i,t}}{U_t} + \frac{\xi}{|\sD|}\1_\mathrm{i \in \sD}, i=1,\ldots,N \label{eq:2}
\end{align}
where $\xi$ is the exploration rate and $U_t:= \sum_{i=1}^{N}{u_{i,t}}$.  $\sD$ represents the dominating set of $\gG$, and $|\sD|$ denotes the cardinality of  $\sD$. Moreover, $\1_\mathrm{i \in \sD}$ denote the indicator function and its value is $1$ when $i \in \sD$. Let $\sS_t$ denote the subset of kernel indices chosen at time $t$, and $I_t$ denote the index of the kernel drawn according to the PMF $p_t$ in \eqref{eq:2}. Therefore, $i \in \sS_t$ if $i \in \sN_{I_t}^{\text{out}}$, which means that the loss associated with the $i$-th kernel is calculated if the $i$-th kernel is an  out-neighber of the $I_t$-th node. In turn, the RF-based function approximation can be obtained as
\begin{align}
    \hat{f}_{\text{RF}}(\vx_t) = \sum_{i \in \sN_{I_t}^{\text{out}}}{\frac{w_{i,t}}{\sum_{j \in \sN_{I_t}^{\text{out}}}{w_{j,t}}}\hat{f}_{\text{RF},i}(\vx_t)}. \label{eq:3}   
\end{align}
%According to \eqref{eq:2}, \eqref{eq:3}, it can be inferred that a subset of kernels are chosen to perform function approximation such that for each pair $(i,j)$ where $i, j \in \sS_t$ we have $\KL ( \pi_{\kappa_i}(\boldsymbol{\psi}) \Vert \pi_{\kappa_j}(\boldsymbol{\psi}) ) \le \gamma_{i}$. This shows that the learner chooses a subset of kernels that are comparatively similar. The goal of the learner is to choose the subset with minimum cumulative loss in hindsight. 
Furthermore, the importance sampling loss estimate $\ell_{i,t}$  at time instant $t$ is defined as
\begin{align}
\ell_{i,t} = \frac{\gL(\vtheta_{i,t}^\top \rvz_i(\vx_t),y_t)}{q_{i,t}}\1_\mathrm{i \in \sS_t}, i=1,\ldots,N \label{eq:5}
\end{align}
where $q_{i,t}$ is the probability that $i \in \sS_t$ and it can be computed as
\begin{align}
q_{i,t} = \sum_{j \in \sN_{i}^{\text{in}}}{p_{j,t}}. \label{eq:6}
\end{align}
In addition, the importance sampling function approximation estimate $\hat{\ell}_{i,t}$ at time instant $t$ associated with $v_i \in \gV$ is defined as
\begin{align}
    \hat{\ell}_{i,t} = \frac{\gL(\hat{f}_{\text{RF}}(\vx_t),y_t)}{p_{i,t}}\1_\mathrm{I_t=i}. \label{eq:14}
\end{align}
Using the importance sampling loss estimate in \eqref{eq:6}, $\vtheta_{i,t}$ can be updated as
\begin{align}
\vtheta_{i,t+1} &= \vtheta_{i,t}-\eta \nabla \ell_{i,t} = \vtheta_{i,t} - \eta \frac{\nabla \gL(\vtheta_{i,t}^\top \rvz_i(\vx_t),y_t)}{q_{i,t}}\1_\mathrm{i \in \sS_t}, \label{eq:7}
\end{align}
where $\eta$ is the learning rate. Moreover, the multiplicative update is employed to update $w_{i,t}$ and $u_{i,t}$ based on  importance sampling loss estimates in \eqref{eq:5} and \eqref{eq:14} as follows
\begin{subequations} \label{eq:8}
\begin{align}
    w_{i,t+1} &= w_{i,t} \exp(-\eta\ell_{i,t}), i=1,\ldots,N \label{eq:8a} \\
    u_{i,t+1} &= u_{i,t} \exp(-\eta\hat{\ell}_{i,t}), i=1,\ldots,N. \label{eq:8b}
\end{align}
\end{subequations}
The procedure to choose a subset of kernels at each time instant for function approximation is summarized in Algorithm \ref{alg:1}. This algorithm is called SFG-MKL which stands for Similarity Feedback Graph based Multi-Kernel Leanring.
\begin{algorithm}[tb]
	\caption{Similarity Feedback Graph aided Online Multi-Kernel Learning (SFG-MKL)}
	\label{alg:1}
	\begin{algorithmic}
		\STATE {\bfseries Input:}{Kernels $\kappa_{i}$, $i=1,...,N$, learning rate $\eta$, exploration rate $\xi$, number of RFs $D$. }
		\STATE \textbf{Initialize:} $\vtheta_{i,1}=\mathbf{0}$, ${w}_{i,1}=1$, $i=1,...,N$, Construct the feedback graph $\gG$ via \Algref{alg:3}.
		\FOR{$t=1,...,T$}
		\STATE Receive one datum $\vx_{t}$.
		\STATE Draw one of nodes $v_i \in \gV$ according to the PMF ${p}_{t}=({p}_{1,t},...,{p}_{N,t})$ in \eqref{eq:2}.
		\STATE Predict $\hat{f}_{\text{RF}}(\vx_{t})$ via \eqref{eq:3}.
		\STATE Calculate loss ${\gL(\hat{f}_{\text{RF},i}(\vx_{t}),y_t)}$ for all $i \in \sS_t$.
		\STATE Update ${\vtheta}_{i,t+1}$ via \eqref{eq:7}.
		\STATE Update ${w}_{i,t+1}$ and $u_{i,t+1}$ via \eqref{eq:8}.
		\ENDFOR
	\end{algorithmic}
\end{algorithm}

\subsection{Regret Analysis}
This subsection presents the regret analysis of SFG-MKL, where we apply stochastic regret \citep{Hazan2016} to measures the difference between expected cumulative loss of the SFG-MKL and the best function approximant in the hindsight. Let $f^*(.)$ denote the best function approximant in the hindsight which can be obtained as
\begin{subequations} \label{eq:10}
\begin{align}
    {f}^{*}(.) &\in \arg \min_{{f}_{i}^{*},i \in \{1,...,N\}} \sum_{t=1}^{T}{\gL({f}_{i}^{*}(\vx_t),y_t)} \label{eq:10a}\\
	{f}_{i}^{*}(.) & \in \arg \min_{f \in \sH_{i}} \sum_{t=1}^{T}{\gL({f}(\vx_t),y_t)} \label{eq:10b} 
\end{align}
\end{subequations}
Hence, the stochastic regret associated with the algorithm SFG-MKL is defined as
\begin{align}
    \sum_{t=1}^{T}{\E_t[\gL(\hat{f}_{\text{RF}}(\vx_t),y_t)]} - \sum_{t=1}^{T}{\gL(f^*(\vx_t),y_t)} \label{eq:11}
\end{align}
where $\E_t[.]$ denotes the expected value at time instant $t$ given the loss observations in prior times. Furthermore, the performance of SFG-MKL is analyzed under the following assumptions:\\
    \textbf{(as1)} The loss function $\gL(\vtheta_{i,t}^\top \rvz_i(\vx_t),y_t)$ is convex with respect to $\vtheta_{i,t}$ at each time instant $t$.\\
    \textbf{(as2)} For $\vtheta$ in a bounded set $\sT$ which satisfies $\| \vtheta \| \le C_{\vtheta}$ the loss and its gradient are bounded as ${\mathcal{L}(\vtheta_{i,t}^{\top}\rvz_{i}(\vx_{t}),y_t)} \in [0,1]$ and $\| \nabla \mathcal{L}(\vtheta_{i,t}^{\top}\rvz_{i}(\vx_{t}),y_t) \| \le L$, respectively.\\
    \textbf{(as3)} Kernels $\{{\kappa}_{i}\}_{i=1}^{N}$ are shift-invariant, standardized, and bounded. And each datum  $\| \vx_{t} \| \le 1$.
    
The following Theorem presents the upper bound for cumulative stochastic regret of SFG-MKL.
\begin{theorem} \label{th:1} 
Under (as1) and (as2), let $j^* = \arg \min_{\forall j:1 \le j \le N}{\sum_{t=1}^{T}{\gL(f_j^*(\vx_t),y_t)}}$. Then for any $i \in \sN_{j^*}^{\text{in}}$, the stochastic regret of SFG-MKL satisfies
\begin{align}
    & \sum_{t=1}^{T}{\E_t[\gL(\hat{f}_{\text{RF}}(\vx_t),y_t)]} - \sum_{t=1}^{T}{\gL(f^*(\vx_t),y_t)} \nonumber \\ \le & \frac{\ln N|\sN_i^{\text{out}}|}{\eta} + \frac{(1+\epsilon) C^2}{2\eta} + \epsilon LTC \nonumber \\ & + (\xi+\frac{\eta N}{2}-\frac{\eta \xi}{2}) T + \frac{\eta}{2}\sum_{t=1}^{T}{(\frac{1}{\Bar{q}_{i,t}}+\frac{L^2}{q_{j^*,t}})} \label{eq:12}
\end{align}
with probability at least $1-2^8(\frac{\sigma_{j^*}}{\epsilon})^2\exp(-\frac{D\epsilon^2}{4d+8})$ under (as1)-(as3) for any $\epsilon > 0$. Furthermore, $\frac{1}{\Bar{q}_{i,t}} = \sum_{j \in \sN_i^{\text{out}}}{\frac{w_{j,t}}{q_{j,t}W_{i,t}}}$, $C$ is a constant and $\sigma_{j^*}^2$ is the second moment of $\pi_{\kappa_{j^*}}(\boldsymbol{\psi})$.
\end{theorem}
\begin{proof}
The proof is deferred to Appendix \ref{ap:A}.
\end{proof}
The regret bound in (\ref{eq:12})  depends on $\frac{1}{\Bar{q}_{i,t}}$ and $\frac{1}{q_{j^*,t}}$. %If $q_{k,t} = \gO(1)$, $\forall v_k \in \gV$, we have $\frac{1}{\Bar{q}_{i,t}}=\gO(1)$ and $\frac{1}{q_{j^*,t}}=\gO(1)$. In this case, putting $\eta=\epsilon=\xi=\gO(\frac{1}{\sqrt{T}})$ results in a regret bound of $\gO(\sqrt{T})$. However, given $\gG$, it is not guaranteed that $q_{k,t} = \gO(1)$, $\forall v_k \in \gV$. 
Since, there is a self-loop for all $v_k \in \gV$, it can be written that $q_{k,t} \ge p_{k,t}$. In addition, based on \eqref{eq:2}, we can conclude that $p_{k,t} > \frac{\xi}{|\sD|}$, $\forall v_k \in \gV$ and as a result $q_{k,t} > \frac{\xi}{|\sD|}$, $\forall v_k \in \gV$. Therefore, in the worst case where there is a $k \in \sN_i^{\text{out}}$ such that $q_{k,t}=\gO(\frac{\xi}{|\sD|})$, considering $\eta=\epsilon=\gO(\frac{1}{\sqrt{T}})$ and $\xi=\gO(T^\frac{1}{4})$, SFG-MKL can achieve regret bound of $\gO(T^\frac{3}{4})$. Since the regret bound of $\gO(\sqrt{T})$ is more satisfactory than the regret bound of $\gO(T^\frac{3}{4})$, in what follows the structure of the feedback graph is refined at each time instant so that the regret of $\gO(\sqrt{T})$ can be achieved.

\section{Online Multi-Kernel Learning with Graph Refinement}
To begin with, let's define set $\sD_t^\prime$ as
\begin{align}
   \sD_t^\prime:=\left\{ i \bigg|\frac{u_{i,t}}{U_t} \ge \frac{1}{1-\xi}(\beta - \frac{\xi}{N})\right\} \label{eq:15}
\end{align}
where $\beta$ is a pre-selected constant. According to \eqref{eq:2}, it can be inferred that $p_{i,t} \ge \beta, \forall i \in \sD_t^\prime$. Let $\gG_t^\prime=(\gV,\gE_t^\prime)$ be a graph such that $\sD_t^\prime$ is a dominating set of $\gG_t^\prime$. Suppose at each time instant $t$, $\gG_t^\prime$ is employed as the feedback graph instead of $\gG$. In this case,  it is ensured that there is at least one edge from $\sD_t^\prime$ to each $v_i \in \gV \setminus \sD_t^\prime$, i.e., $\sD_t^\prime$ is a dominating set of $\gG_t^\prime$. In this case, we have $q_{i,t} \ge \beta$, $\forall v_i \in \gV$. To this end, at each time instant $t$,  $\gG_t^\prime$ can be constructed based on $\gG$ by expanding $\gE$ to $\gE_t^\prime$ such that $\sD_t^\prime$ would be a dominating set of $\gG_t^\prime$. Specifically, at each time instant $t$, the edge $(d_{i,t},i)$ is appended to $\gE_t^\prime$, if there is not any edge from $\sD_t^\prime$ to $v_i$, where 
\begin{align}
    d_{i,t} = \arg \min_{j \in \sD_t^\prime} \Delta ( \kappa_i(\boldsymbol{\rho}) , \kappa_j(\boldsymbol{\rho}) ). \label{eq:16}
\end{align}
Hence, there is at least one edge from $\sD_t^\prime$ to $v_i \in \gV \setminus \sD_t^\prime$, meaning $\sD_t^\prime$ is a dominating set for $\gG_t^\prime$. Then one of the vertices in $\gV$ is drawn according to the probability mass function (PMF) $p_t$, with
\begin{align}
	p_{i,t} = (1-\xi)\frac{u_{i,t}}{U_t} + \frac{\xi}{|\sD_t^\prime|}\1_\mathrm{i \in \sD_t^\prime}, i=1,\ldots,N. \label{eq:17}
\end{align}
Let $\sN_{i,t}^{\text{out}}$ and $\sN_{i,t}^{\text{in}}$ denote sets of out-neighbors and in-neighbors of $v_i$ in $\gG_t^\prime$, respectively. According to \eqref{eq:15} and \eqref{eq:17}, we have $q_{i,t} \ge \beta$, $\forall v_i\in \gV$ where $q_{i,t} = \sum_{j \in \sN_{i,t}^{\text{in}}}{p_{j,t}}$. The RF-based function approximation can be written as
\begin{align}
    \hat{f}_{\text{RF}}(\vx_t) = \sum_{i \in \sN_{I_t}^{\text{out}}}{\frac{w_{i,t}}{\sum_{j \in \sN_{I_t}^{\text{out}}}{w_{j,t}}}\hat{f}_{\text{RF},i}(\vx_t)}. \label{eq:19}   
\end{align} 
According to \eqref{eq:19}, $\vtheta_{i,t}$, $w_{i,t}$ and $u_{i,t}$ can be updated using \eqref{eq:7}, \eqref{eq:8a} and \eqref{eq:8b}, respectively. The procedure is summarized in \Algref{alg:2}, which is called SFG-MKL-R, and its performance is analyzed in the following Theorem. 
\begin{algorithm}[tb]
	\caption{SFG-MKL with Feedback Graph Refinement (SFG-MKL-R)}
	\label{alg:2}
	\begin{algorithmic}
		\STATE {\bfseries Input:}{Kernels $\kappa_{i}$, $i=1,...,N$, learning rate $\eta>0$, exploration rate $\xi>0$, the number of RFs $D$ and the constant $\beta>0$. }
		\STATE \textbf{Initialize:} $\vtheta_{i,1}=\mathbf{0}$, ${w}_{i,1}=1$, $i=1,...,N$, Construct the feedback graph $\gG$ via \Algref{alg:3}.
		\FOR{$t=1,...,T$}
		\STATE Receive one datum $\vx_{t}$.
		\STATE Set $\gE_t^\prime=\gE$ and obtain $d_{i,t}$, $\forall i \in \gV \setminus \sD_t^\prime$ by \eqref{eq:16}.
		\STATE For all $i \in \gV \setminus \sD_t^\prime$, append $(d_{i,t},i)$ to $\gE_t^\prime$ if $(d_{i,t},i) \notin \gE$.
		\STATE Draw one of nodes $v_i \in \gV$ according to the PMF ${p}_{t}=({p}_{1,t},...,{p}_{N,t})$ in \eqref{eq:17}.
		\STATE Predict $\hat{f}_{\text{RF}}(\vx_{t})$ via \eqref{eq:19}.
		\STATE Calculate loss ${\gL(\hat{f}_{\text{RF},i}(\vx_{t}),y_t)}$ for all $i \in \sS_t$.
		\STATE Update ${\vtheta}_{i,t+1}$ via \eqref{eq:7}.
		\STATE Update ${w}_{i,t+1}$ and $u_{i,t+1}$ via \eqref{eq:8}.
		\ENDFOR
	\end{algorithmic}
\end{algorithm}

\begin{theorem} \label{lem:4}
The stochastic regret of SFG-MKL-R satisfies
\begin{align}
    & \sum_{t=1}^{T}{\E_t[\gL(\hat{f}_{\text{RF}}(\vx_t),y_t)]} - \sum_{t=1}^{T}{\gL(f^*(\vx_t),y_t)} \nonumber \\ \le & \frac{2\ln N}{\eta} + \frac{(1+\epsilon) C^2}{2\eta} + \epsilon LTC + (\xi+\frac{\eta}{2}\frac{L^2+N\beta+1}{\beta}-\frac{\eta \xi}{2}) T \label{eq:18}
\end{align}
with probability at least $1-2^8(\frac{\sigma_{j^*}}{\epsilon})^2\exp(-\frac{D\epsilon^2}{4d+8})$ under (as1)-(as3) for any $\epsilon > 0$ and any $\beta \le \frac{1}{N}$.
\end{theorem}
\begin{proof}
The proof can be found in Appendix \ref{ap:B}.
\end{proof}
According to Theorem \ref{lem:4}, by setting $\eta=\epsilon=\xi=\gO(\frac{1}{\sqrt{T}})$ and $\beta=\gO(1)$ such that $\beta \le \frac{1}{N}$, SFG-MKL-R can achieve regret bound of $\gO(\sqrt{T})$ using the feedback graph $\gG_t^\prime$. However, note that since some edges may be added to $\gG$ to construct $\gG_t^\prime$, using $\gG_t^\prime$ instead of $\gG$ may cause increase in computational complexity of function approximation.

\textbf{Computational Complexity.} Both SFG-MKL and SFG-MKL-R needs to store set of $d$-dimensional vectors $\{\boldsymbol{\psi}_{i,j}\}_{j=1}^{D}$ per kernel in addition to two weighting coefficients $\{w_{i,t}\}_{i=1}^N$ and $\{u_{i,t}\}_{i=1}^N$. Therefore, the memory requirement for both of the algorithms are $\gO(dDN)$. Let in order to construct $\gG$, the number of out-neighbors of each node $v_i \in \gV$ is restricted to $M<N$. In this case, the per-iteration complexity of SFG-MKL including calculation of inner products is $\gO(dDM)$. However, in SFG-MKL-R, due to the graph refinement procedure it is possible that the number of out-neighbors of one node in $\gG_t^\prime$ to be $N$. Therefore, the worst case per-iteration computational complexity of SFG-MKL-R is $\gO(dDN)$. Furthermore, the per-iteration computational complexity of OMKR developed by \citet{Sahoo2014} is $\gO(tdN)$ while the per-iteration computational complexity of RF-based online multi-kernel learning algorithms developed by \citet{Sahoo2019} and \citet{Shen2019} are $\gO(dDN)$. In addition, the per-iteration complexity of OMKL-GF \citep{Ghari2020} is $\gO(dDM+JN)$ where $M$ is the maximum number of kernels in the chosen subset and $J$ is a constant related to the algorithm.

%\textbf{Comparison with Online Learning.} In online learning with expert advice, there is a learner interacts with a set of experts such that at each round of learning the learner choose one of the experts and takes its advice for either decision making or prediction \citep{Cesa-Bianchi2006,Auer2003}. The learner may observe the loss associated with a subset of experts after decision making and in this regard this can be modeled by a graph which is called feedback graph \citep{Mannor2011,Cohen2016,Alon2017}. In both SFG-MKL and SFG-MKL-R, each kernel can be viewed as an expert. However, , there are two major innovative differences compared with  online learning with feedback graph: i) the proposed algorithm constructs and refines the feedback graph to improve the performance of learning task while in online learning, the feedback graph is generated in an adversarial manner. ii) in this paper, each expert (kernel) is a learner itself and experts implement an online scheme for self-improvement. 

\textbf{Comparison with OMKL-GF \citep{Ghari2020}.} OMKL-GF by \citet{Ghari2020} also views kernels as experts.
OMKL-GF exploits a bipartite graph without considering the relationships between kernels. In contrast, our novel SFG-MKL and SFG-MKL-R construct feedback graphs based on  similarities between kernels. This feedback graph structure helps SFG-MKL-R to achieve tighter regret of $\gO(\sqrt{T})$ compared with  $\gO(T^\frac{5}{6})$ of OMKL-GF.

\section{Experiments}
This section presents experimental results over real datasets downloaded from UCI Machine Learning Repository \citep{Dua2019}. The performance of the proposed algorithms SFG-MKL and SFG-MKL-R are compared with the following kernel learning benchmarks:\\
\textbf{OMKR}: online multi-kernel regression approach proposed by \citet{Sahoo2014}.\\
 %\textbf{RFOMKR}: online multi-kernel regression based on RF approximation \citep{Sahoo2019}. \\
 \textbf{Raker}: RF-based online multi-kernel learning \citep{Shen2019}.\\
\textbf{OMKL-GF}: online multi-kernel learning with bipartite graph \citep{Ghari2020}.

%%%%%%%%%%%%%%%%

\begin{table}[t]
\caption{MSE $\&$ run time for Airfoil Self-Noise dataset.}
\label{table:1}
\begin{center}
\begin{tabular}{lll}
\toprule
\multicolumn{1}{c}{\bf Algorithms}  &\multicolumn{1}{c}{\bf MSE($\times 10^{-3}$)}  &\multicolumn{1}{c}{\bf Run time (s)}
\\ \hline \\
OMKR    &$28.32$  &$1003.80$ \\
%RFOMKR  &$355.63$    &$2.49$ \\
Raker   &$22.85$ &$4.86$ \\
OMKL-GF &$30.41$ &$2.71$\\
SFG-MKL    &$12.83$ &$\mathbf{1.32}$\\
SFG-MKL-R  &$\mathbf{12.81}$ &$1.91$\\
\bottomrule
\end{tabular}
\end{center}
\end{table}

\begin{table}[t]
\caption{MSE and run time  for Concrete Compressive Strength dataset.}
\label{table:2}
\begin{center}
\begin{tabular}{lll}
\toprule
\multicolumn{1}{c}{\bf Algorithms}  &\multicolumn{1}{c}{\bf MSE($\times 10^{-3}$)}  &\multicolumn{1}{c}{\bf Run time (s)}
\\ \hline \\
OMKR    &$30.51$  &$454.87$ \\
%RFOMKR  &$212.76$    &$1.73$ \\
Raker   &$26.02$ &$3.32$ \\
OMKL-GF &$40.69$ &$2.08$\\
SFG-MKL    &$21.56$ &$\mathbf{0.92}$\\
SFG-MKL-R  &$\mathbf{21.38}$ &$1.43$\\
\bottomrule
\end{tabular}
\end{center}
\end{table}

\begin{table}[t]
\caption{MSE and run time for Naval Propulsion Plants dataset.}
\label{table:3}
\begin{center}
\begin{tabular}{lll}
\toprule
\multicolumn{1}{c}{\bf Algorithms}  &\multicolumn{1}{c}{\bf MSE($\times 10^{-3}$)}  &\multicolumn{1}{c}{\bf Run time (s)}
\\ \hline \\
OMKR    &$7.59$  &$39308.53$ \\
%RFOMKR  &$347.15$    &$23.00$ \\
Raker   &$6.82$ &$38.52$ \\
OMKL-GF &$10.50$ &$17.92$\\
SFG-MKL    &$\mathbf{4.35}$ &$\mathbf{8.99}$\\
SFG-MKL-R  &$4.36$ &$10.35$\\
\bottomrule
\end{tabular}
\end{center}
\end{table}

\begin{table}[t]
\caption{MSE and run time  over Wine Quality dataset.}
\label{table:4}
\begin{center}
\begin{tabular}{lll}
\toprule
\multicolumn{1}{c}{\bf Algorithms}  &\multicolumn{1}{c}{\bf MSE($\times 10^{-3}$)}  &\multicolumn{1}{c}{\bf Run time (s)}
\\ \hline \\
OMKR    &$21.95$  &$5342.42$ \\
%RFOMKR  &$251.54$    &$9.07$ \\
Raker   &$21.04$ &$16.34$ \\
OMKL-GF &$23.95$ &$8.60$\\
SFG-MKL    &$\mathbf{20.19}$ &$\mathbf{4.01}$\\
SFG-MKL-R  &$20.20$ &$4.96$\\
\bottomrule
\end{tabular}
\end{center}
\end{table}

%%%%%%%%%%%%%%%%

The accuracy of different approaches are evaluated using mean square error (MSE). Due to the randomness in the random features extracted for function approximation, we average the MSE over $R = 50$ different sets of random features. The MSE at time $t$ can be computed as
$
    \text{MSE}_t = \frac{1}{R}\sum_{r=1}^{R}{ \frac{1}{t}\sum_{\tau=1}^{t}{(\hat{f}_{\text{RF}}(\vx_{\tau})-y_\tau)^2} }.  $
The number of random features $D=50$. The kernel dictionary contains $41$ radial basis function (RBF) kernels. The bandwidth of the $i$-th kernel is as $10^{\sigma_i}$ where $\sigma_i=\frac{i-21}{10}$. %is drawn randomly from a uniform distribution in the interval $[-2,2]$\reminder{did you find any other work choosing sigma in this way?}. 
Parameters $\xi$ and  $\eta$ are set to $\frac{1}{\sqrt{T}}$ for all algorithms. The performance of kernel learning algorithms is evaluated through several real datasets: \\
\textbf{Airfoil Self-Noise}:  comprises $1503$ different size airfoils at various wind tunnel speeds and each data sample $\vx_t$ includes $5$ features such as frequency and chord length. The output $y_t$ is scaled sound pressure level in decibels \citep{Brooks1989}.\\
\textbf{Concrete Compressive Strength}:  contains $1030$ samples of $8$ features, such as the amount of cement or water in a concrete. The goal is to predict concrete compressive strength \citep{Yeh1998}.\\
\textbf{Naval Propulsion Plants }:  contains $11934$ samples of $15$ features of a naval vessel characterized by a gas turbine propulsion plant including the ship speed and gas turbine shaft torque. The goal is to predict the lever position \citep{Coraddu2016}.\\
\textbf{Wine Quality}: contains $4898$ samples of $11$ features for white wine including the amount of acidity, sugar and alcohol. The goal is to predict the quality of wine \citep{Cortez2009}. 
%\textbf{QSAR Fish Toxicity}:  contains $908$ samples with $6$  molecular descriptors as features. The goal is to predict acute aquatic toxicity towards the fish Pimephales promelas (Cassotti, 2015).

Parameter $\lambda$ is set to $10^{-3}$ and the greedy set cover algorithm by \citet{Chvatal1979} is employed to find the dominating set $\sD$ of the feedback graph $\gG$. In addition, to determine the value of $\gamma_i$ for each vertex $v_i \in \gV$, the number of out-neighbors for each node is set to be $5$.
%In other words, the value of $\gamma_i$ is obtained as $\gamma_i = \min\{\gamma_i||\sN_i^\text{out}|=5\}, i=1,\ldots,N$. 
For SFG-MKL-R, at each time,  $\beta$ is set to $\beta = (1-\xi)\Bar{u}_{[10,t]}+\frac{\xi}{N}$ where $\Bar{u}_{[10,t]}$ denote the tenth greatest value in the sequence $\{\frac{u_{i,t}}{U_t}\}_{i=1}^{N}$. In order to speed up SFG-MKL and SFG-MKL-R, after $300$ time instants, both algorithms choose $I_t=\arg\max_{i}{\frac{u_{i,t}}{U_t}}$. For OMKL-GF, the number of selective nodes is set to be $1$ and the maximum number of kernels chosen at each time instant is $10$, which obtains the best results as shown by \citet{Ghari2020}. All experiments were carried out using Intel(R) Core(TM) i7-10510U CPU @ 1.80 GHz 2.30 GHz processor with a 64-bit {Windows} operating system.

Tables \ref{table:1}, \ref{table:2}, \ref{table:3} and \ref{table:4} show MSE and run time, respectively for kernel learning algorithms when the performance of the algorithms tested on Airfoil Self-Noise, Concrete Compressive Strength, Naval Propulsion Plants and Wine Quality datasets, respectively. As it can be seen from Tables \ref{table:1}, \ref{table:2}, \ref{table:3} and \ref{table:4}, both SFG-MKL and SFG-MKL-R can provide lower MSE in comparison with other algorithms. Also, results in Tables \ref{table:1}, \ref{table:2}, \ref{table:3} and \ref{table:4} illustrate that SFG-MKL runs faster than other kernel learning approaches significantly. From the results in Tables \ref{table:1}, \ref{table:2}, \ref{table:3} and \ref{table:4}, it can be inferred that using the feedback graph $\gG$, SFG-MKL can effectively prune the irrelevant kernels to improve the accuracy of the learning task. Also, run times reported in Tables \ref{table:1}, \ref{table:2}, \ref{table:3} and \ref{table:4} show that choosing a subset of kernels instead of employing all kernels in the given dictionary can reduce the computational complexity of kernel learning since both SFG-MKL and SFG-MKL-R run faster than Raker. %and RFOMKR. 
Also, SFG-MKL is faster than OMKL-GF which shows the effectiveness of SFG-MKL in reduction of computational complexity. %Furthermore, results in Tables \ref{table:1} and \ref{table:2} show that although employing KLFG-MKL-R can improve the accuracy of the function approximation,it also leads to higher computational complexity.
Moreover, Figures \ref{fig:1} and \ref{fig:2} illustrate the MSE of MKL algorithms over time for Airfoil Self Noise and Naval Propulsion plants datasets, respectively. It can be observed that our proposed algorithms converge faster than all other alternatives.
\begin{figure}[t]
	\centering
	\includegraphics[width=.75\linewidth]{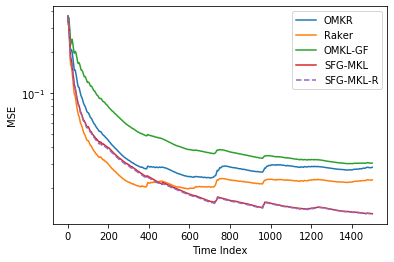}
	\caption{MSE performance on Airfoil Self Noise dataset.}
	\label{fig:1}
\end{figure}
\begin{figure}
	\centering
	\includegraphics[width=.75\linewidth]{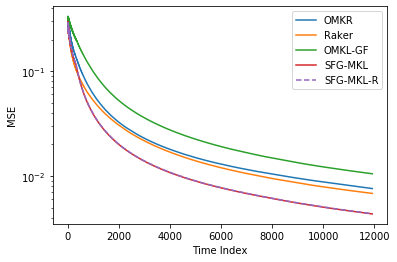}
	\caption{MSE performance on Naval Propulsion Plants dataset.}
	\label{fig:2}
\end{figure}

\section{Conclusion}
This paper presented online multi-kernel learning algorithms for non-linear function learning. Based on the similarities between kernels, a feedback graph was constructed in order to choose a subset of kernels. It was proved that SFG-MKL can achieve sub-linear regret of $\gO(T^\frac{3}{4})$. Furthermore, refining the feedback graph structure at each time instant, SFG-MKL-R was proposed, which enjoys  sub-linear regret of $\gO(\sqrt{T})$. Moreover, experiments on real datasets showed that by choosing a subset of kernels, both SFG-MKL and SFG-MKL-R can obtain lower MSE in comparison with other online multi-kernel learning algorithms OMKR, %RFOMKR, 
Raker and OMKL-GF. Furthermore, experiments showed that SFG-MKL has considerably lower run time compared to OMKR, %RFOMKR, 
Raker and OMKL-GF. %Moreover, run time of SFG-MKL-R is lower than Raker and OMKL-GF.

\bibliographystyle{plainnat}
\bibliography{References}

\appendix

\section{Proof of Lemma \ref{lem:6}} \label{ap:C}
According to \eqref{eq:23}, we obtain
\begin{align}
    \frac{1}{\gU_d}\int |\hat{f}_i(\vx) - \hat{f}_j(\vx)|^2 d\vx  = \frac{1}{\gU_d}\int |\sum_{t=1}^T{\alpha_{i,t}\kappa_i(\vx,\vx_t)} - \sum_{t=1}^T{\alpha_{j,t}\kappa_j(\vx,\vx_t)}|^2 d\vx. \label{eq:a51}
\end{align}
Applying Arithmetic Mean-Geometric Mean inequality on the right hand side of \eqref{eq:a51}, we find
\begin{align}
    & \frac{1}{\gU_d}\int |\hat{f}_i(\vx) - \hat{f}_j(\vx)|^2 d\vx \nonumber \\ \le & \frac{2}{\gU_d} \int \left(|\sum_{t=1}^T{\alpha_{i,t}\left(\kappa_i(\vx,\vx_t)-\kappa_j(\vx,\vx_t)\right)}|^2 + |\sum_{t=1}^T{(\alpha_{j,t}-\alpha_{i,t})\kappa_j(\vx,\vx_t)}|^2\right) d\vx. \label{eq:a52}
\end{align}
Using Cauchy-Schwartz inequality, \eqref{eq:a52} can be further relaxed to
\begin{align}
    & \frac{1}{\gU_d}\int |\hat{f}_i(\vx) - \hat{f}_j(\vx)|^2 d\vx \nonumber \\ \le & \frac{2}{\gU_d} \int (\sum_{t=1}^T{|\alpha_{i,t}|^2})(\sum_{t=1}^T{|\kappa_i(\vx,\vx_t)-\kappa_j(\vx,\vx_t)|^2})d\vx \nonumber \\ & + \frac{2}{\gU_d} \int (\sum_{t=1}^T{|\alpha_{j,t} - \alpha_{i,t}|^2})(\sum_{t=1}^T{|\kappa_j(\vx,\vx_t)|^2})d\vx. \label{eq:a53}
\end{align}
Considering the fact that $C_m:=\max_i \sum_{t=1}^{T}{|\alpha_{i,t}|^2}$, from \eqref{eq:a53} it can be written that
\begin{align}
    & \frac{1}{\gU_d}\int |\hat{f}_i(\vx) - \hat{f}_j(\vx)|^2 d\vx \nonumber \\ \le & \frac{2C_m}{\gU_d}\sum_{t=1}^T{\int |\kappa_i(\vx,\vx_t)-\kappa_j(\vx,\vx_t)|^2 d\vx} + \frac{4C_m}{\gU_d}\sum_{t=1}^T{\int |\kappa_j(\vx,\vx_t)|^2 d\vx}. \label{eq:a54}
\end{align}
Furthermore, based on \eqref{eq:a54} and the fact that $|\kappa_j(\vx,\vx_t)|^2 \le 1$, we can infer that
\begin{align}
    \frac{1}{\gU_d}\int |\hat{f}_i(\vx) - \hat{f}_j(\vx)|^2 d\vx \le \frac{2C}{\gU_d} \sum_{t=1}^T{\left(\Delta(\kappa_i(\vx-\vx_t),\kappa_j(\vx-\vx_t))+2\gU_d\right)} \label{eq:a55}
\end{align}
which proves Lemma \ref{lem:6}.

\section{Proof of Theorem \ref{th:1}} \label{ap:A}
In order to prove Theorem \ref{th:1}, we prove the following intermediate Lemma.
\begin{lemma} \label{lem:1}
Let $\hat{f}_{\text{RF},i}(.)$ denote the sequence of estimates generated by SFG-MKL with a preselected kernel ${\kappa}_{i}$ where $\gF_i = \{\hat{f}_{i}| \hat{f}_{i}(\vx)=\vtheta^\top\rvz_i(\vx), \forall \vtheta \in \R^{2D}\}$. Then, under assumptions (as1) and (as2) the following bound holds
\begin{align}
    \sum_{t=1}^{T}{\gL(\hat{f}_{\text{RF},i}(\vx_t),y_t)} - \sum_{t=1}^{T}{\gL(\hat{f}_{i}^{*}(\vx_t),y_t)} \le \frac{\| \vtheta_i^* \|^2}{2\eta} + \sum_{t=1}^{T}{\frac{\eta L^2}{2 q_{i,t}}} \label{eq:a1}
\end{align}
where $L$ is the Lipschitz constant in (as2) and $\vtheta_i^*$ is the parameter vector associated with the best estimator $\hat{f}_{i}^{*}(\vx) = (\vtheta_{i}^{*})^{\top}\rvz_{i}(\vx)$.
\end{lemma}
\begin{proof}
For $\vtheta_{i,t+1}$ and any fixed $\vtheta$, it can be written that
\begin{align}
    \| \vtheta_{i,t+1} - \vtheta \|^2 & = \| \vtheta_{i,t} - {\eta}{\nabla}\ell_{i,t}  - \mathbold{\theta} \|^2 \nonumber \\ & = \| \mathbold{\theta}_{i,t} - \mathbold{\theta} \|^2 - 2{\eta}{\nabla}^{\top}\ell_{i,t}(\vtheta_{i,t} - \vtheta) + \| {\eta}{\nabla}\ell_{i,t} \|^2. \label{eq:a2}
\end{align}
Furthermore, from the convexity of the loss function with respect to $\vtheta$ in (as1), we can conclude that
\begin{align}
    \gL(\vtheta_{i,t}^\top \rvz_i(\vx_t),y_t) - \gL(\vtheta^\top \rvz_i(\vx_t),y_t) \le \nabla^\top \gL(\vtheta_{i,t}^\top \rvz_i(\vx_t),y_t)(\vtheta_{i,t}-\vtheta). \label{eq:a3}
\end{align}
Therefore, from \eqref{eq:a3}, it can be inferred that
\begin{align}
    & \left(\frac{\gL(\vtheta_{i,t}^\top \rvz_i(\vx_t),y_t)}{q_{i,t}} - \frac{\gL(\vtheta^\top \rvz_i(\vx_t),y_t)}{q_{i,t}}\right)\1_\mathrm{i \in \sS_t} \nonumber \\ \le & \frac{\nabla^\top\gL(\vtheta_{i,t}^\top \rvz_i(\vx_t),y_t)}{q_{i,t}}(\vtheta_{i,t}-\vtheta)\1_\mathrm{i \in \sS_t}. \label{eq:a4}
\end{align}
Based on \eqref{eq:5}, \eqref{eq:a4} is equivalent to
\begin{align}
    \ell_{i,t} - \frac{\gL(\vtheta^\top \rvz_i(\vx_t),y_t)}{q_{i,t}}\1_\mathrm{i \in \sS_t} \le \nabla^\top \ell_{i,t}(\vtheta_{i,t}-\vtheta). \label{eq:a5}
\end{align}
Combining \eqref{eq:a2} with \eqref{eq:a5}, we get
\begin{align}
    \ell_{i,t} - \frac{\gL(\vtheta^\top \rvz_i(\vx_t),y_t)}{q_{i,t}}\1_\mathrm{i \in \sS_t} \le \frac{\| \mathbold{\theta}_{i,t} - \mathbold{\theta} \|^2 - \| \vtheta_{i,t+1} - \vtheta \|^2}{2\eta} + \frac{\eta}{2}\| {\nabla}\ell_{i,t} \|^2. \label{eq:a6}
\end{align}
Taking the expectation of $\ell_{i,t}$ and $\| {\nabla}\ell_{i,t} \|^2$ with respect to $\1_\mathrm{i \in \sS_t}$, we arrive at
\begin{subequations} \label{eq:a7}
\begin{align}
    \E_t[\ell_{i,t}] & = \sum_{j \in \sN_{i}^{\text{in}}}{p_{j,t}\frac{\gL(\vtheta_{i,t}^\top \rvz_i(\vx_t),y_t)}{q_{i,t}}} = \gL(\vtheta_{i,t}^\top \rvz_i(\vx_t),y_t) \label{eq:a7a} \\
    \E_t[\| {\nabla}\ell_{i,t} \|^2] & = \sum_{j \in \sN_{i}^{\text{in}}}{p_{j,t}\frac{\| \nabla \gL(\vtheta_{i,t}^\top \rvz_i(\vx_t),y_t)\|^2}{q_{i,t}^2}} = \frac{\| \nabla \gL(\vtheta_{i,t}^\top \rvz_i(\vx_t),y_t)\|^2}{q_{i,t}}. \label{eq:a7b}
\end{align}
\end{subequations}
Therefore, taking the expectation with respect to $\1_\mathrm{i \in \sS_t}$ from both sides of \eqref{eq:a6}, we obtain
\begin{align}
    & \gL(\vtheta_{i,t}^\top \rvz_i(\vx_t),y_t) - \gL(\vtheta^\top \rvz_i(\vx_t),y_t) \nonumber \\ \le & \frac{\| \mathbold{\theta}_{i,t} - \mathbold{\theta} \|^2 - \| \vtheta_{i,t+1} - \vtheta \|^2}{2\eta} + \frac{\eta \| \nabla \gL(\vtheta_{i,t}^\top \rvz_i(\vx_t),y_t)\|^2}{2 q_{i,t}}. \label{eq:a8}
\end{align}
Based on (as2), we have $\| \nabla \gL(\vtheta_{i,t}^\top \rvz_i(\vx_t),y_t)\|^2 \le L^2$. Therefore, summing \eqref{eq:a8} over time from $t=1$ to $t=T$ it can be concluded that
\begin{align}
    & \sum_{t=1}^{T}{\gL(\vtheta_{i,t}^\top \rvz_i(\vx_t),y_t)} - \sum_{t=1}^{T}{\gL(\vtheta^\top \rvz_i(\vx_t),y_t)} \nonumber \\ \le & \frac{\| \mathbold{\theta} \|^2 - \| \vtheta_{i,T+1} - \vtheta \|^2}{2\eta} + \sum_{t=1}^{T}{\frac{\eta L^2}{2 q_{i,t}}}. \label{eq:a9}
\end{align}
Putting $\vtheta = \vtheta_i^*$ in \eqref{eq:a9} and taking into account that $\| \vtheta_{i,T+1} - \vtheta \|^2 \ge 0$, we can write
\begin{align}
    \sum_{t=1}^{T}{\gL(\vtheta_{i,t}^\top \rvz_i(\vx_t),y_t)} - \sum_{t=1}^{T}{\gL(\vtheta^\top \rvz_i(\vx_t),y_t)} \le \frac{\| \vtheta_i^* \|^2}{2\eta} + \sum_{t=1}^{T}{\frac{\eta L^2}{2 q_{i,t}}} \label{eq:a10}
\end{align}
which completes the proof of Lemma \ref{lem:1}.
\end{proof}
Furthermore, in order to prove Theorem \ref{th:1}, the following intermediate Lemma is also proved.
\begin{lemma} \label{lem:2}
The following inequality holds under (as1) and (as2)
\begin{align}
    \sum_{t=1}^{T}{\E_t[\gL(\hat{f}_\text{RF}(\vx_t),y_t)]} - \sum_{t=1}^{T}{\gL_i(\hat{f}_{\text{RF}}(\vx_t),y_t)} \le \frac{\ln N}{\eta} + \eta(1+\frac{N}{2}-\frac{\eta}{2}) T \label{eq:a11}
\end{align}
where $\gL_i(\hat{f}_{\text{RF}}(\vx_t),y_t)$ denote the loss of function approximation when $v_i$ is drawn.
\end{lemma}
\begin{proof}
For any $t$, we can write
\begin{align}
    \frac{U_{t+1}}{U_t} = \sum_{i=1}^{N}{\frac{u_{i,t+1}}{U_t}} = \sum_{i=1}^{N}{\frac{u_{i,t}}{U_t}\exp(-\eta\hat{\ell}_{i,t})}. \label{eq:a12}
\end{align}
Based on \eqref{eq:2}, we have $\frac{u_{i,t}}{U_t} = \frac{p_{i,t}-\frac{\xi}{|\sD|}\1_\mathrm{i \in \sD}}{1-\xi}$ and as a result \eqref{eq:a12} can be rewritten as
\begin{align}
    \frac{U_{t+1}}{U_t} = \sum_{i=1}^{N}{\frac{p_{i,t}-\frac{\xi}{|\sD|}\1_\mathrm{i \in \sD}}{1-\xi}\exp(-\eta\hat{\ell}_{i,t})}. \label{eq:a13}
\end{align}
Using the inequality $e^{-x} \le 1-x+\frac{1}{2}x^2, \forall x \ge 0$ and \eqref{eq:a13}, it can be concluded that
\begin{align}
    \frac{U_{t+1}}{U_t} \le \sum_{i=1}^{N}{\frac{p_{i,t}-\frac{\xi}{|\sD|}\1_\mathrm{i \in \sD}}{1-\xi}\left(1-\eta\hat{\ell}_{i,t}+\frac{1}{2}(\eta\hat{\ell}_{i,t})^2\right)}. \label{eq:a14}
\end{align}
Taking logarithm from both sides of inequality \eqref{eq:a14},  and use the fact that $1+x \le e^x$, we arrive at
\begin{align}
    \ln \frac{U_{t+1}}{U_t} \le \sum_{i=1}^{N}{\frac{p_{i,t}-\frac{\xi}{|\sD|}\1_\mathrm{i \in \sD}}{1-\xi}\left(-\eta\hat{\ell}_{i,t}+\frac{1}{2}(\eta\hat{\ell}_{i,t})^2\right)}. \label{eq:a15}
\end{align}
Summing \eqref{eq:a15} over $t$ leads to
\begin{align}
    \ln \frac{U_{T+1}}{U_1} \le \sum_{t=1}^{T}{\sum_{i=1}^{N}{\frac{p_{i,t}-\frac{\xi}{|\sD|}\1_\mathrm{i \in \sD}}{1-\xi}\left(-\eta\hat{\ell}_{i,t}+\frac{1}{2}(\eta\hat{\ell}_{i,t})^2\right)}}. \label{eq:a16}
\end{align}
Furthermore, $\ln \frac{U_{T+1}}{U_1}$ can be bounded from below as
\begin{align}
    \ln \frac{U_{T+1}}{U_1} \ge \ln \frac{u_{i,T+1}}{U_1} = - \sum_{t=1}^{T}{\eta \hat{\ell}_{i,t}} - \ln N \label{eq:a17}
\end{align}
for any $i$ such that $1 \le i \le N$. Combining \eqref{eq:a16} with \eqref{eq:a17}, we obtain
\begin{align}
    & \sum_{t=1}^{T}{\sum_{i=1}^{N}{\frac{p_{i,t}\eta}{1-\xi}\hat{\ell}_{i,t}}} - \sum_{t=1}^{T}{\eta \hat{\ell}_{i,t}} \nonumber \\ \le & \ln N + \sum_{t=1}^{T}{\sum_{i=1}^{N}{\frac{\eta \xi \1_\mathrm{i \in \sD}}{(1-\xi)|\sD|}\hat{\ell}_{i,t}}} + \sum_{t=1}^{T}{\sum_{i=1}^{N}{\frac{p_{i,t}-\frac{\xi}{|\sD|}\1_\mathrm{i \in \sD}}{1-\xi}\left(\frac{1}{2}(\eta\hat{\ell}_{i,t})^2\right)}}. \label{eq:a18}
\end{align}
Multiplying both sides by $\frac{1-\xi}{\eta}$ it can be concluded that
\begin{align}
    & \sum_{t=1}^{T}{\sum_{i=1}^{N}{p_{i,t}\hat{\ell}_{i,t}}} - \sum_{t=1}^{T}{\hat{\ell}_{i,t}} \nonumber \\ \le & \frac{\ln N}{\eta} + \sum_{t=1}^{T}{\sum_{i=1}^{N}{\frac{\xi \1_\mathrm{i \in \sD}}{|\sD|}\hat{\ell}_{i,t}}} + \sum_{t=1}^{T}{\sum_{i=1}^{N}{\frac{\eta(p_{i,t}-\frac{\xi}{|\sD|}\1_\mathrm{i \in \sD})}{2}\hat{\ell}_{i,t}^2}}. \label{eq:a19}
\end{align}
In addition, taking the expectation of $\hat{\ell}_{i,t}$ and $\hat{\ell}_{i,t}^2$, we get
\begin{subequations} \label{eq:a20}
\begin{align}
    \E_t[\hat{\ell}_{i,t}] &= p_{i,t}\frac{\gL_i(\hat{f}_{\text{RF}}(\vx_t),y_t)}{p_{i,t}} = \gL_i(\hat{f}_{\text{RF}}(\vx_t),y_t) \label{eq:a20a} \\
    \E_t[\hat{\ell}_{i,t}^2] &= {p_{i,t}\frac{\gL_i^2(\hat{f}_{\text{RF}}(\vx_t),y_t)}{p_{i,t}^2}} = \frac{\gL_i^2(\hat{f}_{\text{RF}}(\vx_t),y_t)}{p_{i,t}} \le \frac{1}{p_{i,t}} \label{eq:a20b}
\end{align}
\end{subequations}
Thus, taking the expectation from both sides of \eqref{eq:a19} leads to
\begin{align}
    & \sum_{t=1}^{T}{\sum_{i=1}^{N}{p_{i,t}\gL(\hat{f}_{\text{RF}}(\vx_t),y_t)}} - \sum_{t=1}^{T}{\gL_i(\hat{f}_{\text{RF}}(\vx_t),y_t)} \nonumber \\ \le & \frac{\ln N}{\eta} + \sum_{t=1}^{T}{\sum_{i=1}^{N}{\frac{\xi \1_\mathrm{i \in \sD}}{|\sD|}\gL(\hat{f}_{\text{RF}}(\vx_t),y_t)}} + \sum_{t=1}^{T}{\sum_{i=1}^{N}{\frac{\eta(p_{i,t}-\frac{\xi}{|\sD|}\1_\mathrm{i \in \sD})}{2p_{i,t}}}}. \label{eq:a21}
\end{align}
Taking into account that $\gL(\hat{f}_{\text{RF}}(\vx_t),y_t) \le 1$ and based on \eqref{eq:a21} we can write
\begin{align}
    & \sum_{t=1}^{T}{\sum_{i=1}^{N}{p_{i,t}\gL(\hat{f}_{\text{RF}}(\vx_t),y_t)}} - \sum_{t=1}^{T}{\gL_i(\hat{f}_{\text{RF}}(\vx_t),y_t)} \nonumber \\ \le & \frac{\ln N}{\eta} + \xi T + \sum_{t=1}^{T}{\sum_{i=1}^{N}{\frac{\eta(p_{i,t}-\frac{\xi}{|\sD|}\1_\mathrm{i \in \sD})}{2p_{i,t}}}}. \label{eq:a22}
\end{align}
Moreover, using \eqref{eq:a22} and the fact that $p_{i,t} \le 1$, it can be concluded that
\begin{align}
    & \sum_{t=1}^{T}{\sum_{i=1}^{N}{p_{i,t}\gL(\hat{f}_{\text{RF}}(\vx_t),y_t)}} - \sum_{t=1}^{T}{\gL_i(\hat{f}_{\text{RF}}(\vx_t),y_t)} \nonumber \\ \le & \frac{\ln N}{\eta} + (\xi+\frac{\eta N}{2}-\frac{\eta \xi}{2}) T. \label{eq:a23}
\end{align}
Furthermore, the expected loss incurred by SFG-MKL given observed losses in prior time instants can be expressed as
\begin{align}
    \E_t[\gL(\hat{f}_\text{RF}(\vx_t),y_t)] &= \sum_{i=1}^{N}{p_{i,t} \gL(\sum_{j \in \sN_{i,t}^{\text{out}}}{\frac{w_{j,t}}{\sum_{k \in \sN_{i,t}^{\text{out}}}{w_{k,t}}}\hat{f}_{\text{RF},j}(\vx_t)},y_t)} \nonumber \\ & = \sum_{i=1}^{N}{p_{i,t}\gL(\hat{f}_{\text{RF}}(\vx_t),y_t)}. \label{eq:a24}
\end{align}
Therefore, from \eqref{eq:a23} and \eqref{eq:a24}, it can be inferred that
\begin{align}
    \sum_{t=1}^{T}{\E_t[\gL(\hat{f}_\text{RF}(\vx_t),y_t)]} - \sum_{t=1}^{T}{\gL_i(\hat{f}_\text{RF}(\vx_t),y_t)} \le \frac{\ln N}{\eta} + (\xi+\frac{\eta N}{2}-\frac{\eta \xi}{2}) T \label{eq:a25}
\end{align}
which establishes the Lemma \ref{lem:2}.
\end{proof}
Furthermore, to prove Theorem \ref{th:1}, we prove the following Lemma.
\begin{lemma} \label{lem:3}
For any $v_i \in \gV$ and any $j \in \sN_i^{\text{out}}$, it can be written that
\begin{align}
    \sum_{t=1}^{T}{\gL_i(\hat{f}_\text{RF}(\vx_t),y_t)} - \sum_{t=1}^{T}{\gL(\hat{f}_\text{RF,j}(\vx_t),y_t)} \le \frac{\ln |\sN_i^{\text{out}}|}{\eta} + \frac{\eta}{2}\sum_{t=1}^{T}{\frac{1}{\Bar{q}_{i,t}}} \label{eq:a27}
\end{align}
where $\frac{1}{\Bar{q}_{i,t}} = \sum_{j \in \sN_i^{\text{out}}}{\frac{w_{j,t}}{q_{j,t}W_{i,t}}}$.
\end{lemma}
\begin{proof}
Let $W_{i,t} = \sum_{j \in \sN_i^{\text{out}}}{w_{j,t}}$. For $v_i \in \gV$ we find
\begin{align}
    \frac{W_{i,t+1}}{W_{i,t}} = \sum_{j \in \sN_i^{\text{out}}}{\frac{w_{j,t+1}}{W_{i,t}}} = \sum_{j \in \sN_i^{\text{out}}}{\frac{w_{j,t}}{W_{i,t}}\exp(-\eta \ell_{j,t})}. \label{eq:a28}
\end{align}
The following inequality can be obtained using the inequality $e^{-x} \le 1-x+\frac{1}{2}x^2, \forall x \ge 0$ as follows
\begin{align}
    \frac{W_{i,t+1}}{W_{i,t}} \le \sum_{j \in \sN_i^{\text{out}}}{\frac{w_{j,t}}{W_{i,t}}\left(1-\eta \ell_{j,t} + \frac{1}{2}(\eta \ell_{j,t})^2\right)}. \label{eq:a29}
\end{align}
Taking the logarithm from both sides of \eqref{eq:a29} and using the inequality $1+x\le e^x$, we get
\begin{align}
    \ln \frac{W_{i,t+1}}{W_{i,t}} \le \sum_{j \in \sN_i^{\text{out}}}{\frac{w_{j,t}}{W_{i,t}}\left(-\eta \ell_{j,t} + \frac{1}{2}(\eta \ell_{j,t})^2\right)}. \label{eq:a30}
\end{align}
Summing \eqref{eq:a30} over time, we obtain
\begin{align}
    \ln \frac{W_{i,T+1}}{W_{i,1}} \le \sum_{t=1}^{T}{\sum_{j \in \sN_i^{\text{out}}}{\frac{w_{j,t}}{W_{i,t}}\left(-\eta \ell_{j,t} + \frac{1}{2}(\eta \ell_{j,t})^2\right)}}. \label{eq:a31}
\end{align}
Moreover, for any $j \in \sN_i^{\text{out}}$, $\ln \frac{W_{i,T+1}}{W_{i,1}}$ can be bounded from below as
\begin{align}
    \ln \frac{W_{i,T+1}}{W_{i,1}} \ge \ln \frac{w_{j,T+1}}{W_{i,1}} = - \sum_{t=1}^{T}{\eta \ell_{j,t}} - \ln |\sN_i^{\text{out}}| \label{eq:a32}
\end{align}
Combining \eqref{eq:a31} with \eqref{eq:a32}, it can concluded that
\begin{align}
    \sum_{t=1}^{T}{\sum_{j \in \sN_i^{\text{out}}}{\frac{w_{j,t}}{W_{i,t}}\ell_{j,t}}} - \sum_{t=1}^{T}{\ell_{j,t}} \le \frac{\ln |\sN_i^{\text{out}}|}{\eta} + \frac{\eta}{2}\sum_{t=1}^{T}{\sum_{j \in \sN_i^{\text{out}}}{\frac{w_{j,t}}{W_{i,t}}\ell_{j,t}^2}}. \label{eq:a33}
\end{align}
For the expected value of $\ell_{i,t}$ and $\ell_{i,t}^2$, we have
\begin{subequations} \label{eq:a34}
\begin{align}
    \E_t[\ell_{j,t}] & = \sum_{k \in \sN_i^{\text{out}}}{p_{k,t}\frac{\gL(\vtheta_{j,t}^\top \rvz_j(\vx_t),y_t)}{q_{j,t}}} = \gL(\vtheta_{j,t}^\top \rvz_j(\vx_t),y_t) \label{eq:a34a} \\
    \E_t[\ell_{j,t}^2] & = \sum_{k \in \sN_i^{\text{out}}}{p_{k,t}\frac{\gL^2(\vtheta_{j,t}^\top \rvz_j(\vx_t),y_t)}{q_{j,t}^2}} = \frac{\gL^2(\vtheta_{j,t}^\top \rvz_j(\vx_t),y_t)}{q_{j,t}} \le \frac{1}{q_{j,t}} \label{eq:a34b}
\end{align}
\end{subequations}
Taking the expectation from \eqref{eq:a33}, we get
\begin{align}
    & \sum_{t=1}^{T}{\sum_{j \in \sN_i^{\text{out}}}{\frac{w_{j,t}}{W_{i,t}}\gL(\vtheta_{j,t}^\top \rvz_j(\vx_t),y_t)}} - \sum_{t=1}^{T}{\gL(\vtheta_{j,t}^\top \rvz_j(\vx_t),y_t)} \nonumber \\ & \le \frac{\ln |\sN_i^{\text{out}}|}{\eta} + \frac{\eta}{2}\sum_{t=1}^{T}{\sum_{j \in \sN_i^{\text{out}}}{\frac{w_{j,t}}{q_{j,t}W_{i,t}}}}. \label{eq:a35}
\end{align}
Let $\frac{1}{\Bar{q}_{i,t}} = \sum_{j \in \sN_i^{\text{out}}}{\frac{w_{j,t}}{q_{j,t}W_{i,t}}}$ which is the weighted sum of $\frac{1}{q_{j,t}}$ such that $j \in \sN_i^{\text{out}}$. Furthermore, according to \eqref{eq:3}, the loss $\gL_i(\hat{f}_\text{RF}(\vx_t),y_t)$ can be written as
\begin{align}
    \gL_i(\hat{f}_\text{RF}(\vx_t),y_t) = \gL(\sum_{j \in \sN_{i}^{\text{out}}}{\frac{w_{j,t}}{W_{i,t}}\hat{f}_{\text{RF},j}(\vx_t)},y_t). \label{eq:a36}
\end{align}
Based on the Jensen's inequality $\gL_i(\hat{f}_\text{RF}(\vx_t),y_t)$ can be bounded from above as
\begin{align}
    \gL_i(\hat{f}_\text{RF}(\vx_t),y_t) \le \sum_{j \in \sN_{i}^{\text{out}}}{\frac{w_{j,t}}{W_{i,t}}\gL(\hat{f}_{\text{RF},j}(\vx_t),y_t)}. \label{eq:a37}
\end{align}
Using \eqref{eq:a35} and \eqref{eq:a37}, we can conclude that
\begin{align}
    \sum_{t=1}^{T}{\gL_i(\hat{f}_\text{RF}(\vx_t),y_t)} - \sum_{t=1}^{T}{\gL(\vtheta_{j,t}^\top \rvz_j(\vx_t),y_t)} \le \frac{\ln |\sN_i^{\text{out}}|}{\eta} + \frac{\eta}{2}\sum_{t=1}^{T}{\frac{1}{\Bar{q}_{i,t}}} \label{eq:a38}
\end{align}
which proves the Lemma \ref{lem:3}.
\end{proof}
Combining Lemma \ref{lem:2} with Lemma \ref{lem:3}, for any $v_j \in \gV$ and any $i \in \sN_j^{\text{in}}$ we obtain
\begin{align}
    & \sum_{t=1}^{T}{\E_t[\gL(\hat{f}_\text{RF}(\vx_t),y_t)]} - \sum_{t=1}^{T}{\gL(\hat{f}_{\text{RF},j}(\vx_t),y_t)} \nonumber \\ & \le \frac{\ln N|\sN_i^{\text{out}}|}{\eta} + (\xi+\frac{\eta N}{2}-\frac{\eta \xi}{2}) T + \frac{\eta}{2}\sum_{t=1}^{T}{\frac{1}{\Bar{q}_{i,t}}}. \label{eq:a26}
\end{align}
In addition, combining Lemma \ref{lem:1} with \eqref{eq:a26}, for any $v_j \in \gV$ and any $i \in \sN_j^{\text{in}}$ we can write
\begin{align}
    & \sum_{t=1}^{T}{\E_t[\gL(\hat{f}_\text{RF}(\vx_t),y_t)]} - \sum_{t=1}^{T}{\gL(\hat{f}_{j}^{*}(\vx_t),y_t)} \nonumber \\ & \le \frac{\ln N|\sN_i^{\text{out}}|}{\eta} + \frac{\| \vtheta_j^* \|^2}{2\eta} + (\xi+\frac{\eta N}{2}-\frac{\eta \xi}{2}) T + \frac{\eta}{2}\sum_{t=1}^{T}{(\frac{1}{\Bar{q}_{i,t}}+\frac{L^2}{q_{j,t}})} \label{eq:a39}
\end{align}
We use the above inequality as a step-stone to prove Theorem \ref{th:1}.

For a given shift invariant kernel $\kappa_{i}$, the maximum point-wise error of the random feature kernel approximant is uniformly bounded with probability at least $1-2^8(\frac{\sigma_i}{\epsilon})^2\exp(-\frac{D\epsilon^2}{4d+8})$, by \citet{Rahimi2007}
\begin{align}
    \sup_{\vx_{j},\vx_{k}\in \gX}{| \rvz_{i}^\top(\vx_{j}) \rvz_{i}(\vx_{k}) - \kappa_{i}(\vx_{j},\vx_{k}) | } < \epsilon \label{eq:a40}
\end{align}
where $\sigma_i^2$ is the second moment of the Fourier transform of $\kappa_i$. Therefore, under (a3) this implies that $\sup_{\vx_{\tau},\vx_{t}\in \gX} \rvz_{i}^{\top}(\vx_{\tau})\rvz_{i}(\vx_{t}) \le 1+\epsilon$ holds with probability at least $1-2^8(\frac{\sigma_i}{\epsilon})^2\exp(-\frac{D\epsilon^2}{4d+8})$. Let $C := \max_{n} \sum_{t=1}^{T}{|{\alpha}_{n,t}^{*}|}$. Hence, $\| \vtheta_j^* \|^2$ can be bounded from above as
\begin{align}
    \| \vtheta_j^* \|^2 \le \|\sum_{t=1}^{T}{\alpha_{j,t}^{*}\rvz_{j}(\vx_t)\|^{2}} \le |\sum_{t=1}^{T}{\sum_{\tau=1}^{T}{\alpha_{j,t}^{*}\alpha_{j,\tau}^{*}\rvz_{j}^{\top}(\vx_t)\rvz_{j}(\vx_\tau)|}} \le (1+\epsilon) C^2 \label{eq:a41}
\end{align}
with probability at least $1-2^8(\frac{\sigma_i}{\epsilon})^2\exp(-\frac{D\epsilon^2}{4d+8})$. Moreover, using the triangle inequality yields
\begin{align}
    & \left| \sum_{t=1}^{T}{\gL(\hat{f}_{j}^{*}(\vx_{t}),y_t)} - \sum_{t=1}^{T}{\gL(f_j^{*}(\vx_{t}),y_t)} \right| \nonumber \\ \le & \sum_{t=1}^{T}\left|{\gL(\hat{f}_{j}^{*}(\vx_{t}),y_t)} - {\gL(f_j^{*}(\vx_{t}),y_t)} \right|. \label{eq:a42}
\end{align}
According to Lipschitz continuity of the loss function, it can be concluded that
\begin{align}
    & \sum_{t=1}^{T}\left|{\gL(\hat{f}_{j}^{*}(\vx_{t}),y_t)} - {\gL(f_j^{*}(\vx_{t}),y_t)} \right| \nonumber \\ & \le \sum_{t=1}^{T}{L \left| \sum_{\tau=1}^{T}{{\alpha}_{j,\tau}^{*}\rvz_{j}^{\top}(\vx_{\tau})\rvz_{n}(\vx_{t})}-\sum_{\tau=1}^{T}{{\alpha}_{j,\tau}^{*}\kappa_{j}(\vx_{\tau},\vx_{t})} \right| }. \label{eq:a43}
\end{align}
Using the Cauchy-Schwartz inequality, we can write
\begin{align}
	& \sum_{t=1}^{T}{L \left| \sum_{\tau=1}^{T}{{\alpha}_{j,\tau}^{*}\rvz_{j}^{\top}(\vx_{\tau})\rvz_{n}(\vx_{t})}-\sum_{\tau=1}^{T}{{\alpha}_{j,\tau}^{*}\kappa_{j}(\vx_{\tau},\vx_{t})} \right| } \nonumber \\ & \le \sum_{t=1}^{T}{ L\sum_{\tau=1}^{T}{|{\alpha}_{j,\tau}^{*}| \left| \rvz_{j}^{\top}(\vx_{\tau})\rvz_{j}(\vx_{t}) - \kappa_{j}(\vx_{\tau},\vx_{t}) \right|} }. \label{eq:a44}
\end{align}
Using \eqref{eq:a40} and \eqref{eq:a44}, we can conclude that the inequality
\begin{align}
    \sum_{t=1}^{T}{\gL(\hat{f}_{j}^{*}(\vx_{t}),y_t)} - \sum_{t=1}^{T}{\gL(f_j^{*}(\vx_{t}),y_t)} \le \epsilon LTC \label{eq:a45}
\end{align}
holds with probability at least $1-2^8(\frac{\sigma_i}{\epsilon})^2\exp(-\frac{D\epsilon^2}{4d+8})$. Therefore, combining \eqref{eq:a39} with \eqref{eq:a41} and \eqref{eq:a45}, it can be inferred that the following inequality
\begin{align}
    & \sum_{t=1}^{T}{\E_t[\gL(\hat{f}_\text{RF}(\vx_t),y_t)]} - \sum_{t=1}^{T}{\gL(f_{j}^{*}(\vx_t),y_t)} \nonumber \\ & \le \frac{\ln N|\sN_i^{\text{out}}|}{\eta} + \frac{(1+\epsilon) C^2}{2\eta} + \epsilon LTC + (\xi+\frac{\eta N}{2}-\frac{\eta \xi}{2}) T + \frac{\eta}{2}\sum_{t=1}^{T}{(\frac{1}{\Bar{q}_{i,t}}+\frac{L^2}{q_{j,t}})} \label{eq:a46}
\end{align}
holds for any $v_j \in \gV$ and any $i \in \sN_{j}^{\text{in}}$ with probability at least $1-2^8(\frac{\sigma_i}{\epsilon})^2\exp(-\frac{D\epsilon^2}{4d+8})$. This completes the proof of Theorem \ref{th:1}.

\section{Proof of Theorem \ref{lem:4}} \label{ap:B}
According to Theorem \ref{th:1}, the following inequality
\begin{align}
    & \sum_{t=1}^{T}{\E_t[\gL(\hat{f}_{\text{RF}}(\vx_t),y_t)]} - \sum_{t=1}^{T}{\gL(f^*(\vx_t),y_t)} \nonumber \\ & \le \frac{\ln N|\sN_i^{\text{out}}|}{\eta} + \frac{(1+\epsilon) C^2}{2\eta} + \epsilon LTC + (\xi+\frac{\eta N}{2}-\frac{\eta \xi}{2}) T + \frac{\eta}{2}\sum_{t=1}^{T}{(\frac{1}{\Bar{q}_{i,t}}+\frac{L^2}{q_{j^*,t}})} \label{eq:a47}
\end{align}
holds with probability at least $1-2^8(\frac{\sigma_{j^*}}{\epsilon})^2\exp(-\frac{D\epsilon^2}{4d+8})$ for any $\epsilon > 0$ and any $i \in \sN_{j^*}^{\text{in}}$. When $\gG_t^\prime$ is generated by \Algref{alg:2} as the feedback graph, $\sD_t^\prime$ is a dominating set of $\gG_t^\prime$. Furthermore, $k \in \sD_t^\prime$ if $p_{k,t} \ge \beta$. Based on \eqref{eq:6}, if $i \in \sN_{k,t}^{\text{in}}$ (i.e. in-neighborhood of $v_k$ in $\gG_t^\prime$), $q_{k,t} \ge p_{i,t}$. Also, considering the condition $\beta \le \frac{1}{N}$, it is ensured that $\sD_t^\prime$ is not an empty set. Moreover, each node $v_k$ in $\gG_t^\prime$ is out-neighbor of at least one node in $\sD_t^\prime$. Thus, we can conclude that $q_{k,t} \ge \beta$, $\forall v_k \in \gV$. Hence, it can be written that
\begin{align}
    \sum_{t=1}^{T}{(\frac{1}{\Bar{q}_{i,t}}+\frac{L^2}{q_{j^*,t}})} \ge \sum_{t=1}^{T}{(\sum_{j \in \sN_i^{\text{out}}}{\frac{w_{j,t}}{W_{i,t}\beta}}+\frac{L^2}{\beta})} = \frac{(L^2+1)T}{\beta}. \label{eq:a48}
\end{align}
Furthermore, since $|\sN_i^{\text{out}}| \le N$, we have $N|\sN_i^{\text{out}}| \le N^2$. Combining \eqref{eq:a47} with \eqref{eq:a48}, it can be inferred that in this case the stochastic regret of SFG-MKL-R satisfies
\begin{align}
    & \sum_{t=1}^{T}{\E_t[\gL(\hat{f}_{\text{RF}}(\vx_t),y_t)]} - \sum_{t=1}^{T}{\gL(f^*(\vx_t),y_t)} \nonumber \\ & \le \frac{2\ln N}{\eta} + \frac{(1+\epsilon) C^2}{2\eta} + \epsilon LTC + (\xi+\frac{\eta}{2}\frac{L^2+N\beta+1}{\beta}-\frac{\eta \xi}{2}) T \label{eq:a50}
\end{align}
with probability at least $1-2^8(\frac{\sigma_{j^*}}{\epsilon})^2\exp(-\frac{D\epsilon^2}{4d+8})$ under (as1)-(as3) for any $\epsilon > 0$ and any $\beta \le (1-\xi)\max_{k}\frac{u_{k,t}}{U_t}+ \frac{\xi}{N}$. This completes the proof of Theorem \ref{lem:4}.

\vfill

\end{document}

%% file: math_commands.tex
\usepackage{amsmath,amsfonts,bm}

\def\Algref#1{Algorithm~\ref{#1}}

\def\1{\bm{1}}

\def\rvz{{\mathbf{z}}}

\def\vtheta{{\bm{\theta}}}

\def\vx{{\bm{x}}}

\DeclareMathAlphabet{\mathsfit}{\encodingdefault}{\sfdefault}{m}{sl}
\SetMathAlphabet{\mathsfit}{bold}{\encodingdefault}{\sfdefault}{bx}{n}

\def\gC{{\mathcal{C}}}

\def\gE{{\mathcal{E}}}
\def\gF{{\mathcal{F}}}
\def\gG{{\mathcal{G}}}

\def\gL{{\mathcal{L}}}

\def\gO{{\mathcal{O}}}

\def\gU{{\mathcal{U}}}
\def\gV{{\mathcal{V}}}

\def\gX{{\mathcal{X}}}

\def\sD{{\mathbb{D}}}
\def\sH{{\mathbb{H}}}

\def\sN{{\mathbb{N}}}

\def\sR{{\mathbb{R}}}
\def\sS{{\mathbb{S}}}
\def\sT{{\mathbb{T}}}

\newcommand{\E}{\mathbb{E}}

\newcommand{\R}{\mathbb{R}}